\documentclass{article}

\PassOptionsToPackage{numbers,sort&compress}{natbib}
\usepackage[preprint]{neurips_2026}

\usepackage[utf8]{inputenc}
\usepackage[T1]{fontenc}

\usepackage{booktabs}
\usepackage{amsmath}
\usepackage{amssymb}
\usepackage{amsfonts}
\usepackage{amsthm}
\usepackage{enumitem}
\usepackage{graphicx}
\usepackage{microtype}
\usepackage{xcolor}
\usepackage{tikz}
\usetikzlibrary{arrows.meta, positioning, shapes.geometric, fit, backgrounds, calc}
\usepackage{pgfplots}
\pgfplotsset{compat=1.16}
\usepackage{hyperref}
\usepackage{url}
\usepackage{cleveref}

\hypersetup{colorlinks=true, linkcolor=black, citecolor=black, urlcolor=black,
  pdftitle={Split the Labor: Separating Evidence Interpretation from Decision Aggregation},
  pdfauthor={Zhelun Wu}}

\newtheorem{proposition}{Proposition}
\newtheorem{corollary}{Corollary}
\theoremstyle{remark}
\newtheorem{remark}{Remark}

\newcommand{\evidence}{\mathcal{E}}
\newcommand{\sources}{\mathcal{S}}
\newcommand{\labels}{\mathcal{Y}}
\newcommand{\llr}{\ell}

\title{Split the Labor: Separating Evidence Interpretation from Decision
Aggregation --- A Design Principle, Its Failure Mode, and Two Instantiations}

\author{%
  Zhelun Wu \\
  Atlassian \\
  \texttt{zhelunwu@hotmail.com} \\
}

\begin{document}
\maketitle

\begin{abstract}
Systems that ask a language model to reach a conclusion from many sources
usually concatenate them into one prompt. This conflates two operations with
different requirements. Interpreting a source rewards capacity and
context. Combining interpretations rewards fixed arithmetic, comparability
across instances, and the option to return nothing. Once separated, the design
problem becomes the interface between them. We propose a four-field
evidence tuple (hypothesis, reliability bucket, rationale, provenance) and show
that fixing it determines both halves. The separation also reveals a
failure mode in how such systems combine, which we call \emph{count-scale
drift}. Thresholding a sum of unnormalized weights is exactly posterior
thresholding, but at an operating point that slides with the number of sources
consulted. The
slide grows with reader reliability. When source reliabilities differ, the vote
rule and the posterior order instances differently, and no threshold reconciles
them. Pooling calibrated log-likelihood ratios addresses both problems. The fix
is arithmetic rather than architectural, and applies to a class of rules beyond
language models: score-summing triage engines, diagnostic panels scored by
counting positives, and additive multi-signal detectors. We then instantiate the
principle twice on one longitudinal corpus, once after outcomes resolve and once
before. The same partition helps in both, at different granularities: over
reading in the first, over learning capacity in the second. There, a small
sequence encoder on an easy auxiliary objective plus a tree ensemble carrying
the censored survival loss reaches 0.921 AUPRC against 0.805 for a hand-crafted
baseline. We separate what transfers from what must be
re-estimated per domain, and state five predictions that would falsify the
framework, three negative results, and which comparisons remain confounded.
\end{abstract}

\section{Introduction}

Decision-support systems increasingly ask language models to interpret
unstructured evidence distributed across many partially informative and
sometimes conflicting sources. The prevailing architecture concatenates
them into one prompt. This is simple and often effective, but it has four
costs. Provenance is lost, and recovering attribution after the fact is itself
an open problem~\cite{rashkin2023,gao2023alce}. Evidence quality is
heterogeneous, so unweighted joint reading lets many uninformative sources
dominate a few decisive ones. Long contexts degrade retrieval of mid-context
information~\cite{liu2024lost}. And a generative system asked for a conclusion
supplies one regardless of support.

These four failures share a cause. A monolithic reader is being asked to do
two jobs whose requirements point in opposite directions. Interpretation
rewards capacity, context, and the freedom to weigh nuance; combination
rewards fixed arithmetic, comparability across instances, calibrated
strength, and the discipline to return nothing. Asking one mechanism for
both means each job is done at the other's terms, and the interface between
them --- where provenance, reliability, and uncertainty would live --- never
gets designed at all.

\paragraph{The principle.} Partition the inference at that interface, then
design the interface first. Concretely: read each source in isolation under
a causal context constraint and require it to emit a fixed \emph{evidence
tuple}; then combine tuples with an explicit, auditable rule. Fixing the
contract determines the rest. Four things follow. Reliability is declared rather than implied, so weighting
becomes an estimation problem instead of a tuning problem. Provenance is
structural rather than reconstructed. Failures are isolated to a single reading
operation. And abstention becomes expressible: a rule over tuples can return no
conclusion, where a generative reader asked for one will produce it anyway.
The partition also generalizes past reading: in our second study the same
cut is applied to \emph{learning capacity}, with a small encoder taking the
easy sub-problem and a tree ensemble the hard one.

\paragraph{Why the combination step needs attention.} Systems that do
partition usually combine by summing weights and thresholding the sum.
\Cref{sec:agg} shows that this is not a different rule from posterior
thresholding. It is the same rule, evaluated at an operating point that drifts
with the number of sources. The drift rate depends on parameters that are
rarely estimated in practice. The result is stated generally enough to cover a
class of rules used well beyond language models. It also comes with a
diagnostic that can be computed
on data these systems already log.

None of the mechanisms we compose is new. Selective prediction with a
reject option dates to Chow~\cite{chow1970} with a modern
treatment~\cite{elyaniv2010,geifman2017}. Reliability-weighted aggregation
of noisy annotators originates with Dawid and
Skene~\cite{dawid1979,whitehill2009,raykar2010}. Weight of
evidence~\cite{good1950}, opinion pooling~\cite{genest1986}, evidential
combination~\cite{shafer1976}, and noisy-OR~\cite{pearl1988} address
combination from different starting points. Task
decomposition~\cite{khot2023} and cross-sample
aggregation~\cite{wang2023sc} are established prompting techniques.

What is new is the composition rather than any mechanism: which cut to make,
what contract to place at the cut, which combination rules that contract
admits, and what fails when an unsuitable rule is used.

\begin{itemize}[leftmargin=*]
\item \textbf{A design principle with an interface.} \Cref{sec:iface} states
the evidence tuple, the four axes along which a task determines its
instantiation (evidence locality, reliability heterogeneity, dependence
structure, resolution regime), and the conditions under which partitioning
pays rather than costs.

\item \textbf{Analytical.} \Cref{sec:agg} names and characterizes
count-scale drift (\Cref{prop:scale}): threshold rules over unnormalized
weight sums are posterior rules with an operating point linear in source
count, biasing toward over-assertion on evidence-rich instances, worse for
more reliable readers. \Cref{prop:redundant} shows the additive
corroboration bonus is decision-equivalent to rescaling the threshold.
Pooling calibrated log-likelihood ratios removes both defects, is exact
under stated conditions (\Cref{prop:exact}). \Cref{prop:hetero} shows the
heterogeneous case is worse than drift alone: unless weights are
proportional to per-bucket log-odds gain, no posterior threshold reproduces
the vote rule at all.

\item \textbf{Methodological.} We do not use elicited model confidence,
since it is poorly calibrated~\cite{kadavath2022,xiong2024,guo2017}. Instead,
reliability buckets are constructed from observable properties of the reading
operation: source length, and whether the reader returned a non-empty grounded
rationale. Their sensitivity and false-positive rates are then estimated from
data. We give a block-structured dependence
discount so sources sharing an upstream origin do not accumulate as
independent corroboration, and we pool over the reader's \emph{confusion
matrix} rather than per-label rates, because our reader emits one label per
source and per-label pooling is misspecified in that regime
(\Cref{sec:confusion}).

\item \textbf{Two instantiations, one object, two regimes.} The same corpus
is studied after resolution (\Cref{sec:casea}) and before it
(\Cref{sec:caseb}). What changes between them is only whether the label
exists and whether future evidence is admissible; the partition pays in
both, which is stronger evidence for the principle than either study alone.

\item \textbf{Transfer, stated systematically.} \Cref{sec:transfer} maps
the contract onto four other domains along the same four axes, and
separates what is domain-neutral (the contract, the pooling arithmetic, the
drift diagnostic) from what must be re-estimated (rates, blocks,
thresholds).

\item \textbf{A falsifiable programme.} \Cref{sec:predictions} states five
predictions the framework makes and the experiment that discriminates each,
with three negative results and an explicit register of confounds.
\end{itemize}

One boundary is worth drawing early, because it determines how the paper
should be read. The results are measurements of a running system whose
combination step is the weighted vote our analysis criticizes
(\Cref{sec:deployedrule}); the log-likelihood-ratio construction is derived,
not deployed. So the empirical claims here concern the partition and the
reading strategy, and the analytical claims concern the combination rule.
Keeping those separate is what makes each usable: a measured architecture
with a diagnosed defect is more actionable than an unmeasured architecture
without one.

\section{Related work}

\paragraph{Selective prediction.} Chow's reject option~\cite{chow1970}
formalizes the coverage/risk trade-off, developed
in~\cite{elyaniv2010,geifman2017}; it motivates our emission threshold,
though we report no risk--coverage curve. Conformal
prediction~\cite{angelopoulos2021} offers distribution-free coverage
guarantees and is a natural alternative to our threshold pair; we discuss
why we did not adopt it in \Cref{sec:discussion}.

\paragraph{Aggregating unreliable sources.} Dawid and
Skene~\cite{dawid1979} estimate per-annotator error rates from unlabeled
data by EM; \cite{whitehill2009,raykar2010} extend this. Our per-bucket
rates are a coarsened Dawid--Skene model in which ``annotator'' identity is
the reader's confidence bucket rather than a person, which keeps the
parameter count small at our label frequencies. Unlike Dawid--Skene we
estimate supervised, which trades label efficiency for stability; the
unsupervised variant is one of our baselines.

\paragraph{Decomposition and attribution.}
\cite{wei2022,khot2023,wang2023sc} decompose reasoning but aggregate by
majority vote over samples of one context, not over independent evidence
sources. RAG~\cite{lewis2020} supplies sources and usually concatenates.
Attribution metrics~\cite{rashkin2023,gao2023alce,min2023} and
consistency-based detection~\cite{manakul2023} evaluate groundedness after
generation; we make provenance a structural output.

\paragraph{Survival analysis and tabular learning.} We work in discrete
time rather than in the continuous-time proportional-hazards
tradition~\cite{cox1972}; competing-risk formulations
follow~\cite{finegray1999,wang2019survey}; neural
approaches include~\cite{katzman2018,lee2018deephit} and tree-based
approaches~\cite{ishwaran2008}. That a boosted tree~\cite{chen2016xgboost}
carrying the survival loss beats an end-to-end neural model is consistent
with trees remaining competitive on tabular data at moderate
scale~\cite{grinsztajn2022,shwartzziv2022}.

\section{The evidence interface}
\label{sec:iface}

\subsection{The contract}

A reading operation returns
$e_i = (\hat{Y}_i, b_i, r_i, a_i)$: a hypothesis over the label space, an
ordinal reliability bucket, a rationale, and provenance. Each of the four fields removes a specific
failure.

The hypothesis makes the reader's commitment explicit rather than leaving it
embedded in prose. The bucket carries reliability as data, so combination
weights can be estimated from outcomes instead of chosen by hand. It is an
ordinal bucket rather than a probability, because elicited model confidence is
poorly calibrated~\cite{kadavath2022,xiong2024,guo2017}. Our buckets are
therefore built from observable properties of the reading operation itself. In
the deployment studied here, two such properties are used: whether the source
is long enough to be informative, and whether the reader returned a grounded
rationale. The rationale field makes that judgement auditable, and is one of
the properties the bucket is derived from. Provenance makes attribution
structural, so no post-hoc attribution model is needed to answer why a label
was asserted.

The contract is deliberately weak about what a reader \emph{is}. Anything
that emits the tuple composes: a language model over text, a statistical
detector over telemetry, a vision model over images, a human reviewer. What
they must share is a calibration scale, not an architecture --- which is the
mechanism behind the transfer claims in \Cref{sec:transfer}.
\Cref{fig:arch} shows the partition and the interface it creates.

\begin{figure}[t]
\centering
\resizebox{\textwidth}{!}{%
\begin{tikzpicture}[
  font=\small,
  box/.style={draw, rounded corners=1.5pt, align=center, inner sep=3.5pt},
  src/.style={box, fill=black!4, minimum width=1.5cm, minimum height=0.52cm},
  rd/.style={box, fill=black!8, minimum width=1.75cm, minimum height=0.52cm},
  tup/.style={box, fill=white, minimum width=2.5cm, minimum height=0.52cm},
  agg/.style={box, fill=black!8, align=center, minimum width=2.5cm},
  dec/.style={box, minimum width=1.5cm, minimum height=0.44cm},
  ar/.style={-{Latex[length=1.6mm]}, thin},
  dar/.style={-{Latex[length=1.4mm]}, thin, dashed, black!55}]

\node[src] (s1) {$s_1$};
\node[src, below=3.5mm of s1] (s2) {$s_2$};
\node[src, below=3.5mm of s2] (sn) {$s_n$};

\node[rd, right=7mm of s1] (r1) {read $s_1$};
\node[rd, right=7mm of s2] (r2) {read $s_2$};
\node[rd, right=7mm of sn] (rn) {read $s_n$};

\node[tup, right=7mm of r1] (e1) {$e_1=(\hat Y_1,b_1,r_1,a_1)$};
\node[tup, right=7mm of r2] (e2) {$e_2=(\hat Y_2,b_2,r_2,a_2)$};
\node[tup, right=7mm of rn] (en) {$e_n=(\hat Y_n,b_n,r_n,a_n)$};

\foreach \i in {1,2,n} { \draw[ar] (s\i) -- (r\i); \draw[ar] (r\i) -- (e\i); }
\draw[dar] (s1.east) to[out=-35,in=150] node[midway, above right=-1mm and -2mm,
  font=\scriptsize, black!60] {causal context} (r2.west);
\draw[dar] (s2.east) to[out=-35,in=150] (rn.west);

\node[agg, right=8mm of e2, text width=3.3cm] (pool)
  {LLR pooling\\[1pt] \scriptsize $\alpha_{b,y},\beta_{b,y}$ rates\\
   \scriptsize $\kappa_j$ dependence discount\\ \scriptsize $\pi(y\mid z)$ prior};
\foreach \i in {1,2,n} { \draw[ar] (e\i.east) -- (pool.west); }

\node[dec, right=7mm of pool, yshift=6.5mm] (as) {assert};
\node[dec, right=7mm of pool] (ab) {abstain};
\node[dec, right=7mm of pool, yshift=-6.5mm] (rj) {reject};
\foreach \d in {as,ab,rj} { \draw[ar] (pool.east) -- (\d.west); }

\begin{scope}[on background layer]
  \draw[black!35, dashed] ($(e1.north east)+(3.5mm,3mm)$) --
    ($(en.south east)+(3.5mm,-3mm)$);
\end{scope}
\node[font=\scriptsize\itshape, black!60, align=center,
  above=1mm of e1.north east, xshift=6mm] {the interface};

\node[font=\scriptsize, black!60, below=2.5mm of rn.south, xshift=2mm]
  {interpretation: capacity, context};
\node[font=\scriptsize, black!60, below=2.5mm of pool.south]
  {combination: calibrated arithmetic};
\end{tikzpicture}}
\caption{The partition. Each source is read alone under a causal context
constraint and must emit the same four-field evidence tuple; a fixed rule
then combines tuples into a decision that includes abstention. The dashed
line is the design surface: everything left of it can be replaced without
touching the arithmetic, and everything right of it can be corrected without
re-reading a source.}
\label{fig:arch}
\end{figure}
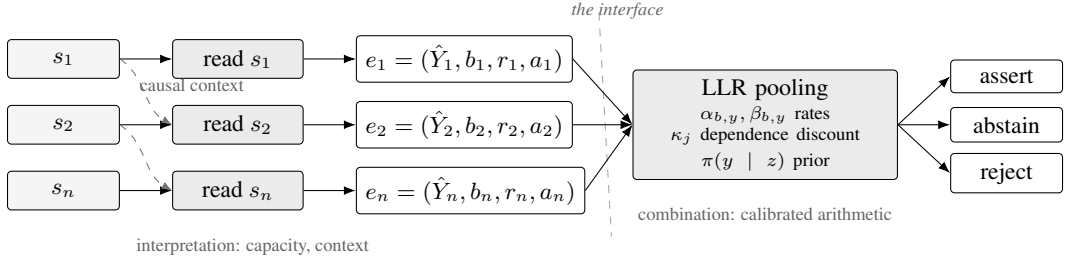

\subsection{Four axes that determine an instantiation}

Given a task, four properties fix most design choices.

\begin{enumerate}[leftmargin=*]
\item \emph{Evidence locality.} Is a source interpretable alone, or only
against its predecessors? This sets how much context each reading operation
gets, and whether the context is causal or complete.
\item \emph{Reliability heterogeneity.} How widely does source
informativeness vary, and is that variation observable before reading? Wide
and observable variation is what makes buckets worth carrying.
\item \emph{Dependence structure.} Do sources share upstream origins?
Blocks of shared origin must not accumulate as independent corroboration
(\Cref{sec:dep}).
\item \emph{Resolution regime.} Has the outcome resolved? If so, an
empirical prior is available and backfilling a record is meaningful; if
not, the prior is unavailable and the output is a live read, not a
committed conclusion.
\end{enumerate}

Our two studies differ on exactly one of these --- the fourth --- which is
why we treat them as one experiment on the principle rather than two
unrelated applications.

\subsection{When partitioning pays}

Partitioning is not free: it forfeits joint reasoning across sources, adds
an estimation problem for the rates, and requires the interface to be
maintained as prompts evolve. Four conditions make it worthwhile. The
source count varies enough that comparability across instances matters.
Reliability is heterogeneous enough that weighting changes decisions.
Provenance is an obligation rather than a convenience. And abstention has
value, meaning a wrong committed answer costs more than no answer. Where
sources are few, homogeneous, and uncontested, concatenation is the better
architecture and this paper does not apply.

\subsection{Formal setting}

An instance comprises sources $\sources=\{s_1,\dots,s_n\}$ and metadata
$z$. The target is a label set $Y \subseteq \labels$. We formulate
multi-label because our metrics are multi-label; a single $\arg\max$ is
inconsistent with reporting instance-level $F_1$.

\paragraph{An asymmetry we must carry through the formalism.} Ground truth
is a set, but our deployed reader emits \emph{at most one} label per
source, $|\hat{Y}_i| \le 1$, with a distinguished value for ``insufficient
evidence''. The multi-label target is therefore assembled entirely by
aggregation. Two consequences follow, and both affect the aggregation. First, one
source's assertions across labels are mutually exclusive by construction,
so they are negatively dependent given the label and cannot be pooled
label-by-label as if independent (\Cref{sec:confusion}). Second, silence
about $y$ is never uninformative: a source that asserted a competing label
was asked and answered otherwise. We keep the general notation $\hat{Y}_i$
because nothing in the aggregation requires $|\hat{Y}_i| \le 1$, but every
number we report comes from the single-label regime.

Each source yields an evidence object
\begin{equation}
e_i = f_\theta(s_i, c_i) = (\hat{Y}_i,\, b_i,\, r_i,\, a_i),
\end{equation}
with $c_i$ the causally available context, $\hat{Y}_i$ the asserted labels,
$b_i \in \{1,\dots,B\}$ an ordinal reliability bucket, $r_i$ the rationale,
$a_i$ provenance. Per label $y$, an aggregator returns $\hat{p}(y)$ and a
decision in $\{\text{assert},\text{reject},\text{abstain}\}$.

\subsection{Independent reading}

Each source is read alone, with a bounded summary of causally prior sources
and never future ones. This yields traceability (each assertion maps to one
reading operation with rationale attached), bounded context (cost scales
with source count, and late sources are not truncated because earlier ones
consumed the window~\cite{liu2024lost}), and failure isolation. The cost argument is arithmetic rather than measured:
reading each source once is linear in source count, where concatenation
grows super-linearly in prompt length. We report no wall-clock or token
measurements and make no latency claim.

The causal constraint is a deployment requirement: a system scoring a live
instance cannot see future sources. Full-context reading is therefore an
oracle reference only.

Note that causal context \emph{induces} dependence between evidence
objects: reader $i$ conditions on a summary of sources $1..i-1$, so
assertions are not conditionally independent given the label. This is the
motivation for the dependence discount in \Cref{sec:dep}, not an
afterthought.

\section{Combining evidence: a defect and its repair}
\label{sec:agg}

\subsection{The deployed rule}
\label{sec:deployedrule}

Every number in \Cref{sec:casea} comes from the following rule, and we
state it first so that the analysis below is read as criticism of a running
system rather than as a description of one. Each source receives a weight from a
fixed table. A source contributes zero if the reader returned no
determination, and zero if the source is short and carries any determination
other than an unresponsiveness label. It contributes $0.3$ if it is short and
carries that label, and $0.3$ if it is long enough but its determination
arrives without a grounded rationale. It contributes $1.0$ if it is long enough
and has one. A label is emitted
when its summed weight reaches $1.0$. Ties are broken toward the latest
source, and both labels are surfaced if a tie survives. One label is treated
asymmetrically: a generic unresponsiveness label is admitted on short
sources, where nothing else is, and is then dropped post hoc whenever any
other label survives threshold. So the emitted set is not simply the set
above threshold, and the formalism of \Cref{sec:iface} does not capture this
rule --- a subsumption constraint over labels sits outside it. For resolved
instances a single most-frequent label per source-count bucket is layered
on as a prior fallback, fitted on the 129 calibration instances of
\Cref{sec:eligibility}. The weights and the threshold were set by hand and
never fitted.

Two properties of this rule matter for what follows. Its weights are
bucket-dependent and were not chosen to be proportional to any log-odds
quantity, so it is the heterogeneous case of \Cref{prop:hetero} rather than
the uniform case of \Cref{prop:scale}: it inherits count-scale drift and
adds a misweighting that no threshold can absorb. And its reliability buckets are constructed from observable
properties of the reading operation --- source length and rationale
presence --- not from elicited model confidence, which is the one design
choice here we would keep unchanged.

\subsection{Three defects in weighted voting}

A common rule assigns weights $w_i$ and scores
$V(y)=\sum_i w_i \mathbb{I}[y \in \hat{Y}_i]$, then adds a corroboration
bonus $\lambda h(C(y))$ where $C(y)=\sum_i \mathbb{I}[y \in \hat{Y}_i]$
counts asserting sources and $h$ is increasing, e.g.\ $h(C)=\log(1+C)$.

\begin{proposition}[Redundancy of the corroboration bonus]
\label{prop:redundant}
If weights are uniform, $w_i = w > 0$, then for any $\lambda \ge 0$ and any
strictly increasing $h$, the decision rule $\{S(y) \ge \tau\}$ with
$S(y) = V(y) + \lambda h(C(y))$ is identical to a threshold rule on the
vote count alone: there exists $\tau'$ with
$\{S(y) \ge \tau\} = \{C(y) \ge \tau'\}$.
\end{proposition}
\begin{proof}
Under uniform weights $V(y) = w\,C(y)$, so
$S(y) = w\,C(y) + \lambda h(C(y)) =: \phi(C(y))$ where
$\phi(c) = wc + \lambda h(c)$ is strictly increasing in $c$ as a sum of
strictly increasing and non-decreasing functions. A strictly increasing
$\phi$ on the integers admits a threshold correspondence: set
$\tau' = \min\{c \in \mathbb{Z}_{\ge 0} : \phi(c) \ge \tau\}$, with
$\tau' = \infty$ when the set is empty (the rule never asserts). Then
$\phi(C) \ge \tau \iff C \ge \tau'$.
\end{proof}

\begin{remark}
\Cref{prop:redundant} is elementary --- any strictly increasing function of
a scalar statistic induces the same family of threshold rules --- and we
state it as a proposition only because the bonus term is widely implemented
as though it changed the decision. The substantive content of this section
is \Cref{prop:scale}.
\end{remark}

So under uniform weights the bonus induces no decision that rescaling the
threshold could not, and $\lambda$ is unidentifiable from decisions alone.
The bonus has an effect only when weights are heterogeneous, and it then trades
weight against count with no principle fixing the exchange rate between them.
\Cref{prop:hetero} supplies that exchange rate. Once weights are set
accordingly, the bonus has nothing left to do.

Before the general statement, it helps to see the mechanism on a small case.
Suppose three sources are read and two assert $y$. Now suppose six are read
and the same two assert $y$. A vote rule cannot distinguish these cases: both
have a vote score of two. A posterior can. In the second case, four sources
were asked about $y$ and declined it, which is evidence against. The vote rule
therefore treats a case carrying more contrary evidence as identical to one
carrying less. The proposition below makes this exact and gives the rate.

\begin{proposition}[A vote threshold is a posterior threshold that drifts with source count]
\label{prop:scale}
Assume (i) every source is read for every label, so non-assertion is
\emph{informed silence}; (ii) readers share bucket rates $(\alpha,\beta)$
with $\alpha > \beta$; (iii) assertions are conditionally independent given
the label. Write $A = \log\frac{\alpha}{\beta} > 0$ and
$D = \log\frac{1-\alpha}{1-\beta} < 0$. For an instance with $n$ sources of
which $k$ assert $y$, the exact posterior log-odds are
\begin{equation}
\operatorname{logit} \Pr[y \mid \evidence]
= \operatorname{logit}\pi(y) + kA + (n-k)D .
\label{eq:exactpost}
\end{equation}
Then for uniform weights the vote rule $\{V = wk \ge \tau\}$, with
$\tau' = \tau/w$, is \emph{exactly} the posterior rule
$\{\operatorname{logit}\Pr[y \mid \evidence] \ge T(n)\}$ where
\begin{equation}
T(n) = \operatorname{logit}\pi(y) + \tau'(A - D) + nD .
\label{eq:drift}
\end{equation}
The vote rule is thus not \emph{a different} rule from posterior
thresholding; it is posterior thresholding at an operating point that falls
linearly in $n$ at rate $|D| = \big|\log\frac{1-\alpha}{1-\beta}\big|$ per
additional source.
\end{proposition}
\begin{proof}
\Cref{eq:exactpost} is the weight-of-evidence
decomposition~\cite{good1950}: each asserting source contributes $A$ and
each silent source contributes $D$. Regrouping,
$\operatorname{logit}\Pr[y\mid\evidence] = \operatorname{logit}\pi(y) + nD
+ k(A-D)$ with $A - D > 0$, so
$k \ge \tau' \iff \operatorname{logit}\Pr[y\mid\evidence] \ge
\operatorname{logit}\pi(y) + nD + \tau'(A-D) = T(n)$.
\end{proof}

We call this \emph{count-scale drift}. A concrete case may help. Take
$\alpha=0.9$, $\beta=0.2$ and a prior of $0.2$, so that each assertion is worth
$+1.50$ in log-odds and each silence $-2.08$. Consider two instances that a
vote rule cannot distinguish, since both have exactly two asserting sources.
With three sources the posterior log-odds are $-0.46$, or a probability near
$0.39$. With six sources they are $-6.69$, or a probability near $0.001$. The
vote score is identical in both cases; the posteriors differ by roughly a
factor of $300$. \Cref{eq:drift} is more useful than
a bare incomparability claim in three ways. It identifies the direction of the drift ($D<0$, so the bar falls as
evidence accumulates). It gives the rate, which depends on a
parameter that is rarely estimated in practice. The drift is larger for more
reliable readers, since $\alpha \to 1$ sends $|D| \to \infty$. And it makes the problem a matter of degree
rather than of kind. The practical question is then how widely $n$ varies in a
given deployment. Where $n$ is nearly constant, the vote rule is a harmless
reparameterization. Where $n$ ranges over an order of magnitude, as it does
here from single-source instances up to 25, it is not.
\Cref{fig:drift} plots the effective threshold against $n$ at three reader
reliabilities.

\begin{figure}[t]
\centering
\begin{tikzpicture}
\begin{axis}[
  width=0.78\textwidth, height=5.1cm, font=\small,
  xlabel={number of sources consulted, $n$},
  ylabel={implied threshold (log-odds)},
  ylabel style={align=center},
  xmin=0.6, xmax=8.4, ymin=-21, ymax=6,
  xtick={1,...,8}, ytick={5,0,-5,-10,-15,-20},
  grid=major, grid style={black!12},
  legend style={at={(0.98,0.03)}, anchor=south east, draw=black!25,
                font=\scriptsize, fill=white, row sep=1pt},
  tick label style={font=\scriptsize},
  label style={font=\scriptsize}]
\addplot[thick, black, dashed] coordinates {(0.6,0)(8.4,0)};
\addlegendentry{calibrated rule: one threshold, flat in $n$}
\addplot[thick, black, mark=*, mark size=1.4pt] coordinates
  {(1,-0.29)(2,-1.39)(3,-2.48)(4,-3.58)(5,-4.68)(6,-5.78)(7,-6.88)(8,-7.98)};
\addlegendentry{vote rule, $(\alpha,\beta)=(0.75,0.25)$}
\addplot[thick, black!55, mark=square*, mark size=1.3pt] coordinates
  {(1,1.56)(2,-1.39)(3,-4.33)(4,-7.28)(5,-10.22)(6,-13.16)(7,-16.11)(8,-19.05)};
\addlegendentry{vote rule, $(\alpha,\beta)=(0.95,0.05)$}
\end{axis}
\end{tikzpicture}
\caption{Count-scale drift (\Cref{prop:scale}), plotted in posterior
log-odds. A vote threshold is a posterior threshold, but not a fixed one:
the bar it enforces falls linearly in the number of sources consulted, at
rate $\log\frac{1-\alpha}{1-\beta}$. More reliable readers drift faster ---
at $(\alpha,\beta)=(0.95,0.05)$ the effective bar drops by almost three
log-odds per source. Curves are the exact values of \Cref{eq:drift} at
$\tau'=1$ and prior $0.2$; nothing is fitted.}
\label{fig:drift}
\end{figure}
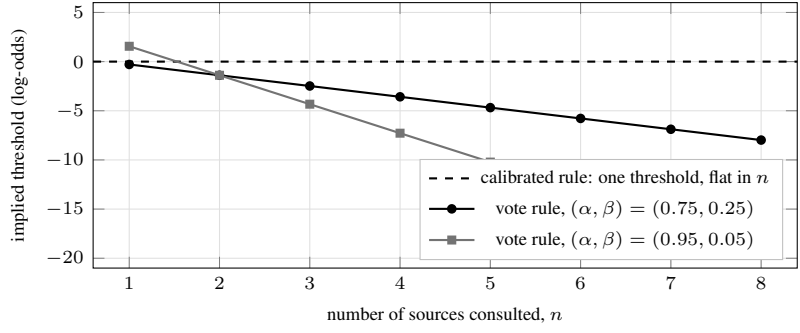

\begin{remark}[Premise (i) is necessary]
\label{rem:premise}
If readers fire only positively and silence is genuinely uninformative,
then $D = 0$, $T(n)$ is constant, and the count-scale defect vanishes; what
remains is only the weighting question of \Cref{prop:hetero}. So the claim
is not that vote counting is always wrong; it is wrong precisely when
non-assertion carries information --- which in our deployment it does, and
strongly, because a single-label reader that named a competing label was
asked about $y$ and declined it (\Cref{sec:confusion}).
\end{remark}

\Cref{prop:scale} assumes homogeneous rates, which our deployment does not
have: its weights are bucket-dependent by design. The heterogeneous case is
the one that matters in practice, and it is worse.

\begin{proposition}[Heterogeneous buckets: misweighting on top of drift]
\label{prop:hetero}
Let bucket $b$ have rates $(\alpha_b,\beta_b)$ with $\alpha_b > \beta_b$,
write $A_b = \log\frac{\alpha_b}{\beta_b}$,
$D_b = \log\frac{1-\alpha_b}{1-\beta_b}$, and
$g_b = A_b - D_b > 0$. Under premises (i) and (iii) of
\Cref{prop:scale},
\begin{equation}
\operatorname{logit}\Pr[y \mid \evidence]
= \operatorname{logit}\pi(y) + \sum_{i=1}^{n} D_{b_i}
+ \!\!\sum_{i:\, y \in \hat{Y}_i}\!\! g_{b_i}.
\label{eq:hetpost}
\end{equation}
Consider the vote rule
$\{V \ge \tau\}$ with $V = \sum_{i: y \in \hat{Y}_i} w_{b_i}$ and $w_b>0$.
\begin{enumerate}[label=(\alph*),leftmargin=*]
\item The vote rule agrees with some posterior threshold rule on every
instance if and only if $w_b = c\,g_b$ for a single constant $c>0$: weights
must be proportional to the per-bucket log-odds gain, not merely ordered
like it.
\item When they are, the induced posterior threshold is
$T = \operatorname{logit}\pi(y) + \sum_{i} D_{b_i} + \tau/c$, which varies
across instances through the \emph{bucket composition} of their sources ---
recovering \Cref{eq:drift} when all buckets coincide.
\item When they are not, no threshold on the posterior reproduces the rule:
the two orderings cross.
\end{enumerate}
\end{proposition}
\begin{proof}
\Cref{eq:hetpost} is \Cref{eq:exactpost} with per-source rates, regrouped as
before. For (a), sufficiency: if $w_b = c\,g_b$ then
$V = c\sum_{i: y \in \hat{Y}_i} g_{b_i}
= c(\operatorname{logit}\Pr[y\mid\evidence] - \operatorname{logit}\pi(y)
- \sum_i D_{b_i})$, so $V \ge \tau$ iff the posterior log-odds exceed the
$T$ in (b). For necessity, suppose $w_b/g_b > w_{b'}/g_{b'}$ for some pair
$b,b'$. Then $w_{b'}/w_b < g_{b'}/g_b$, so the open interval
$(w_{b'}/w_b,\; g_{b'}/g_b)$ is non-empty and contains a rational $m/m'$
with $m,m'$ positive integers. Take instance $\mathcal{A}$ consisting of $m$
asserting sources all in bucket $b$, and $\mathcal{B}$ of $m'$ asserting
sources all in bucket $b'$, with no silent sources and a common prior. Then
$m\,g_b < m'\,g_{b'}$ while $m\,w_b > m'\,w_{b'}$: the posterior ranks
$\mathcal{B}$ above $\mathcal{A}$ and the vote ranks $\mathcal{A}$ above
$\mathcal{B}$. Any $\tau \in (V(\mathcal{B}), V(\mathcal{A})]$ therefore
asserts on the lower-posterior instance and abstains on the higher, which no
monotone rule in the posterior does. This proves (c) and the necessity half
of (a).
\end{proof}

The gap is large at realistic parameters. Take three buckets with
$(\alpha,\beta)$ equal to $(0.55,0.35)$, $(0.75,0.25)$ and $(0.95,0.05)$. The
corresponding gains are $g = 0.82$, $2.20$ and $5.89$. Now apply hand-set
weights of $0.3$, $0.3$ and $1.0$, of the shape deployed systems use. The
weights are correctly \emph{ordered}, but their ratios to $g$ range over a
factor of $2.7$. Sampling instances with up to eight sources, the emitted set
and the withheld set overlap across roughly $17$ units of posterior
log-odds. The rule is therefore not a noisy
approximation to a posterior threshold; it is a different ordering.

\begin{corollary}[Scope of the result]
\label{cor:general}
\Cref{prop:scale} makes no reference to language models, and applies to any
decision rule of the form ``sum weights over the evidence that fired, emit
if the sum exceeds a constant'' where the number of evidence channels
consulted varies across instances and non-firing is informative. That form
includes score-summing rule engines in fraud and security triage,
multi-test diagnostic panels scored by counting positives, and multi-signal
anomaly detectors that add per-signal scores. In all of these, every channel
is evaluated on every instance, so a channel that does not fire has declined
rather than been omitted.

The result does \emph{not} automatically extend to retrieval pipelines. There
an unretrieved passage is absent rather than silent, so $D=0$ may hold and the
drift need not appear; the misweighting of \Cref{prop:hetero} still applies.
Where the result does apply, the remedy is arithmetic rather than
architectural: convert each channel's contribution to a calibrated
log-likelihood ratio before summing, then threshold the posterior.
\end{corollary}

\begin{corollary}[Directional bias]
\label{cor:bias}
Thresholding $V$ asserts at lower posterior on instances with many sources
than on instances with few. Hence, at matched coverage, precision under a
vote rule should decline in $n$, while precision under a rule that
thresholds \Cref{eq:exactpost} should not.
\end{corollary}

\Cref{cor:bias} is a falsifiable prediction about the deployed rule, and it
is testable without new inference, since the per-source evidence objects
are already persisted: stratify instances by source count and compare
per-bin precision against per-bin prevalence. We have not run it. We record
one design constraint for whoever does: label prevalence may itself vary
with $n$, so a prevalence trend would mimic or mask the predicted precision
trend, and the comparison must be against the prevalence-adjusted
expectation rather than a flat line.

\subsection{Per-bucket reliability estimation}

For each label $y$ and bucket $b$ we estimate
\begin{align}
\alpha_{b,y} &= \Pr[\, y \in \hat{Y}_i \mid y \in Y,\; b_i=b \,], &
\beta_{b,y}  &= \Pr[\, y \in \hat{Y}_i \mid y \notin Y,\; b_i=b \,],
\end{align}
smoothed with a Beta$(1,1)$ prior, with rare labels pooled into a fallback
bucket. These are Dawid--Skene error rates~\cite{dawid1979} with buckets in
place of annotator identities.

\paragraph{Estimation data must be disjoint from evaluation data.}
These rates are not fitted in the deployed system, which uses the hand-set
weights of \Cref{sec:deployedrule}; the only fitted component is the
frequency prior $\pi$, estimated on the 129 calibration instances and
evaluated on the disjoint 33 (\Cref{sec:eligibility}). Were the rates
fitted, the same partition would be required: our labels come from a single
curated set, so estimating $\alpha,\beta,\rho$ and evaluating on the same
instances would contaminate every figure.

\subsection{Single-label readers need a confusion matrix, not a rate pair}
\label{sec:confusion}

When $|\hat{Y}_i| \le 1$ the reader is a multinomial channel, and its error
structure is a confusion matrix rather than a pair of per-label rates.
Pooling each label independently from $(\alpha_{b,y}, \beta_{b,y})$ alone
treats ``source $i$ named a competing label $y'$'' identically to ``source
$i$ was silent'', when in fact the first is stronger evidence against $y$
than the second. Define, per bucket,
\begin{equation}
M_b(y' \mid \cdot) = \Pr[\hat{Y}_i = \{y'\} \mid \cdot,\; b_i = b],
\qquad y' \in \labels \cup \{\varnothing\},
\end{equation}
where $\varnothing$ denotes the reader's ``insufficient evidence'' output
and $\cdot$ is the conditioning event $y \in Y$ or $y \notin Y$. Source $i$
contributes to label $y$
\begin{equation}
\llr_i(y) = \log
\frac{M_{b_i}(y'_i \mid y \in Y)}{M_{b_i}(y'_i \mid y \notin Y)},
\label{eq:llrconf}
\end{equation}
positive when $y'_i = y$, typically negative when $y'_i$ is a competing
label, and near zero when $y'_i = \varnothing$ is uninformative about $y$.
A full matrix costs $O(|\labels|^2)$ parameters per bucket, which our label
frequencies do not support. We therefore coarsen the emission to three outcomes
per target label: asserts $y$, asserts some other label, or abstains. This
costs four parameters per (bucket, label), two free per conditioning event,
against two for the rate pair. It preserves the distinction that matters. \Cref{eq:llr} below is the
special case in which the middle outcome is merged into silence. Neither
variant produced our results: the deployed rule counts weighted assertions
and treats every non-assertion, competing or silent, as a non-event, which
is the coarsest of the three.

\subsection{Log-likelihood-ratio pooling}

Source $i$ contributes
\begin{equation}
\llr_i(y) =
\begin{cases}
\log \dfrac{\alpha_{b_i,y}}{\beta_{b_i,y}}, & y \in \hat{Y}_i,\\[1.8ex]
\log \dfrac{1-\alpha_{b_i,y}}{1-\beta_{b_i,y}}, & y \notin \hat{Y}_i,
\end{cases}
\label{eq:llr}
\end{equation}
so informed silence is evidence against $y$ rather than neutral. Pooled
posterior log-odds are
\begin{equation}
\operatorname{logit} \hat{p}(y)
= \operatorname{logit} \pi(y \mid z)
+ \sum_{j=1}^{m} \kappa_j \!\!\sum_{i \in \mathcal{B}_j}\!\! \llr_i(y),
\label{eq:pool}
\end{equation}
where $\mathcal{B}_1,\dots,\mathcal{B}_m$ partition sources into dependence
blocks and $\kappa_j$ discounts within-block accumulation
(\Cref{sec:dep}).

\begin{proposition}[Exactness]
\label{prop:exact}
Suppose (i) assertions are conditionally independent given the label;
(ii) the rates $\alpha_{b,y},\beta_{b,y}$ are exact; (iii) $\kappa_j=1$; and
(iv) the bucket assignment is \emph{ancillary}, meaning
$\Pr[b_i = b \mid y \in Y] = \Pr[b_i = b \mid y \notin Y]$ for all $b$.
Then \Cref{eq:pool} equals the true posterior log-odds of $y$.
\end{proposition}
\begin{proof}
The observation from source $i$ is the pair $(b_i, \mathbb{I}[y \in
\hat{Y}_i])$, whose likelihood factors as
$\Pr[b_i \mid \cdot]\Pr[y \in \hat{Y}_i \mid \cdot, b_i]$. Under (iv) the
first factor cancels in the ratio, leaving exactly $\llr_i(y)$ of
\Cref{eq:llr}; summing over conditionally independent sources and adding the
log-prior-odds is Bayes' rule~\cite{good1950}.
\end{proof}

\begin{remark}[Premise (iv) is not free, and we likely violate it]
\label{rem:ancillary}
Conditioning on the bucket discards the evidence carried by the bucket
\emph{assignment}. When buckets are not ancillary the exact posterior
carries an additional term $\sum_i \log
\frac{\Pr[b_i \mid y \in Y]}{\Pr[b_i \mid y \notin Y]}$, and
\Cref{eq:pool} is a lower-variance approximation rather than an identity.
Source length is plausibly ancillary; rationale presence is plausibly not,
since a reader is likelier to produce a grounded rationale when a genuine
reason exists --- which makes a high bucket itself weak evidence for the
label. The direction of the resulting error is then conservative: pooling
understates the evidence for $y$. Estimating the two bucket-assignment
distributions is cheap on the same calibration fold, and we regard the
omission as an approximation to be measured rather than a licence.
\end{remark}

In words: start from the prior odds, add one term per source, and read off the
result. A source that asserts $y$ contributes a positive term, sized by how
much more often its bucket asserts $y$ correctly than incorrectly. A source
that stays silent contributes a negative term. A source from an uninformative
bucket contributes close to zero, without anyone having to choose a weight for
it.

Three consequences follow. Corroboration is automatic --- two independent asserting
sources contribute twice the log-evidence --- so no bonus term and no
$\lambda$ is needed, and \Cref{prop:redundant} does not apply. The output is
a probability, so one threshold is comparable across instances with
different source counts and different bucket compositions, addressing
\Cref{prop:scale} and \Cref{prop:hetero}. The prior enters on the correct
scale, so no mixing coefficient is needed; where the outcome has not
resolved or prior contamination is a concern, the prior term is dropped and
the pooled sum read as a likelihood ratio.

\subsection{Block-structured dependence discount}
\label{sec:dep}

A single global discount penalizes genuinely independent sources in
large instances. We instead partition sources into blocks sharing an
upstream origin (same author, same document lineage, or adjacency under the
causal-context chain) and discount within block:
\begin{equation}
\kappa_j = \frac{1}{1 + \rho\,(|\mathcal{B}_j| - 1)},
\end{equation}
with $\rho \in [0,1)$ the average residual within-block correlation of
assertions given the label. With $\rho = 0$ this reduces to independent
pooling; as $\rho \to 1$ each block contributes the evidence of a single
source. We fit $\rho$ by maximizing calibration-fold log-likelihood of the
pooled posterior.

We claim no exactness for this correction and it should not be read as one.
It is a design-effect discount, in the spirit of an effective sample size:
exact at the two endpoints, an interpolation in between, and not derived
from any specified within-block dependence model. Its virtue is that it has
one interpretable parameter fitted on held-out data, and it shrinks the
pooled evidence toward the prior in both directions --- a mostly silent
block moves the posterior less far \emph{down}, which is the intended
behaviour and not a sign error. As with the rest of this section, $\rho$ is
unfitted here and the blocking rule untested; whether per-block discounting
earns its complexity over a single global discount is open.

\subsection{Relation to noisy-OR: a different link, not a special case}
\label{sec:noisyor}

Noisy-OR~\cite{pearl1988} is the other natural rule for this setting and is
widely used for it, so the relationship is worth stating precisely. Under
noisy-OR, each asserting source $i$ is treated as an independent sufficient
cause. Source $i$ succeeds with probability $q_i$, and $y$ holds if any
source succeeds:
\begin{equation}
\hat{p}(y) = 1 - \!\!\prod_{i:\, y \in \hat{Y}_i}\!\! (1 - q_i),
\qquad\text{equivalently}\qquad
\log\big(1-\hat{p}(y)\big) = \!\!\sum_{i:\, y \in \hat{Y}_i}\!\!
\log(1-q_i).
\label{eq:noisyor}
\end{equation}

It is tempting --- and we asserted as much in an earlier draft --- to call
this a special case of \Cref{eq:pool}. It is not. Evidence accumulates
additively in \emph{different parameterizations}: noisy-OR is additive in
$\log(1-\hat p)$, a complementary-log link, whereas \Cref{eq:pool} is
additive in $\operatorname{logit}\hat p$. The two families intersect only
trivially. A counterexample suffices: with $q_i = 0.5$ throughout,
\Cref{eq:noisyor} gives posterior log-odds $0, 1.099, 1.946, 2.708$ for
$1,2,3,4$ asserting sources, whose successive increments $1.099, 0.847,
0.762$ are not constant, while \Cref{eq:pool} increments by exactly
$\log(\alpha/\beta)$ per asserting source; no choice of prior and rates
reproduces the sequence.

The substantive difference is saturation. Noisy-OR treats each assertion as
a sufficient cause, so its increments shrink as the posterior approaches
one and it can never be driven to certainty by volume alone --- attractive
when readers are unreliable in unmodelled ways. Log-odds pooling treats each
assertion as a likelihood ratio, so increments are constant and $k$
independent assertions can reach any confidence, which is correct under its
premises and dangerous when they fail. Neither dominates: they encode
different beliefs about what a source \emph{is}. What \Cref{eq:pool} does
offer over \Cref{eq:noisyor} is not generality but scope --- it admits
negative and competing evidence, false-positive rates, an explicit prior,
and a dependence discount, none of which \Cref{eq:noisyor} expresses.

\subsection{Assumption ledger}

\Cref{tab:ledger} collects every result in this section with the assumptions
it needs and the deployment condition that voids it. We include it because
the results are only useful to someone else if the conditions under which
they fail are as legible as the statements.

\begin{table}[t]
\caption{Assumptions and failure conditions for the results of this
section. Conditional independence given the label is assumed by
every result except \Cref{prop:redundant}, and is the premise our own
causal-context design violates by construction --- which is what the
dependence discount is for.}
\label{tab:ledger}
\centering
\small
\begin{tabular}{@{}p{2.9cm}p{4.6cm}p{5.4cm}@{}}
\toprule
Result & Needs & Fails when \\
\midrule
\Cref{prop:redundant} (bonus is redundant) &
Uniform weights; $h$ increasing; $\lambda \ge 0$ &
Weights heterogeneous --- then the bonus does change decisions, though with
no principle fixing the exchange rate \\[0.5ex]
\Cref{prop:scale} (count-scale drift) &
Informed silence; homogeneous rates; conditional independence &
Positive-only firing ($D=0$); or heterogeneous buckets, where
\Cref{prop:hetero} governs \\[0.5ex]
\Cref{prop:hetero} (misweighting) &
Informed silence; conditional independence &
Weights already proportional to $g_b$; or one bucket only \\[0.5ex]
\Cref{cor:general} (scope) &
Every channel consulted per instance &
Channels omitted rather than silent, as in retrieval \\[0.5ex]
\Cref{cor:bias} (precision declines in $n$) &
The above, plus prevalence not varying with $n$ &
Prevalence trends with $n$, which must be measured, not assumed \\[0.5ex]
\Cref{prop:exact} (exactness) &
Conditional independence; exact rates; $\kappa=1$; ancillary buckets &
Any of the four; bucket ancillarity is the one we expect to fail
(\Cref{rem:ancillary}) \\[0.5ex]
$\kappa_j$ discount &
Nothing --- it is a design-effect heuristic &
Always approximate away from $\rho \in \{0,1\}$; no exactness is claimed \\
\bottomrule
\end{tabular}
\end{table}

\subsection{Abstention as deployed}
\label{sec:coverage}

Abstention in the running system is a single one-sided threshold: a label
whose summed weight falls short of $1.0$ is not emitted, and an instance
with no label above threshold produces no committed determination. There is
no reject band and no calibrated $\hat{p}(y)$ to threshold, so the natural
generalization of the proposal in this section --- assert above
$\tau_{\text{hi}}$, reject below $\tau_{\text{lo}}$, abstain between ---
is not what runs.

We report no risk--coverage curve. That absence is not incidental, because
the evaluation population excludes every instance whose ground truth is
``no determinable reason'' (\Cref{sec:eligibility}), so the data needed to
measure what abstention buys has been filtered away upstream of any curve
we could draw. We therefore make no quantitative claim for the abstention
mechanism at all, and treat the coverage question as open rather than as
answered by an unreported experiment.

\section{Experimental setup}

\subsection{Reader configuration}

The reader is a hosted instruction-tuned language model accessed through an
internal gateway, at temperature $0.1$ with a 4{,}000-token output cap and
\emph{one} sample per source: no self-consistency, so the per-source
variance we would need to separate reader noise from aggregation error is
not measured. The prompt is pinned to a dated version identifier and the
aggregation is re-validated when that identifier changes. Context budgets were set from the
corpus length distribution. The target source is capped at 55{,}000 characters,
its 95th percentile being 54{,}718. The two immediately prior sources are
capped at 35{,}000 characters each, roughly the 85th percentile per neighbour.
Older prior sources enter only as summaries and key points, under per-field
caps of 700 and 2{,}500 characters within a 12{,}000 character total. At most
25 sources are processed per instance.
Execution uses 25 concurrent workers with up to three retries; the daily
job reads a four-day rolling window and checkpoints for crash recovery.

\subsection{Reproducibility}

The corpus is proprietary operational data and cannot be released, so
external replication of the numbers is impossible by construction. Everything else can and should be
released. That includes the prompt templates and their version identifier, the
reader model and decoding settings, and the bucket-construction rule. It also
includes the estimation code for $\alpha,\beta,\rho,\pi$, the fold assignment,
the aggregation and evaluation code, and seeds. Finally it includes the metric
definitions, particularly instance $F_1$, which is not micro-$F_1$. We also state the
compute: one reader call per source, no fine-tuning, and a tree model
trained in minutes on a single machine.
Reader model, decoding settings, context budgets and concurrency are stated
in \Cref{sec:casea}'s configuration subsection; prompt templates, the fold
assignment, and seeds are not yet included. A venue that requires a
reproducibility checklist will accept a proprietary corpus and will not
accept missing prompts.

\subsection{Scope of the comparisons we report}

Both instantiations vary one axis and hold the rest fixed. Instantiation~I varies how
sources are read --- jointly, independently, with or without causal context
--- while the aggregation rule stays the deployed weighted vote throughout.
Instantiation~II varies the model architecture and feature set while the reader stays
fixed. No experiment in this paper varies the combination rule, so nothing
here separates the aggregation proposals of \Cref{sec:agg} from the
deployed vote, and we draw no such comparison.

\subsection{Statistical protocol}

All confidence intervals are bootstrap intervals resampling
\emph{instances}, not rows~\cite{efron1979}. This matters because the
forecasting test set contains many snapshots per instance, so row-level
resampling would understate variance by treating correlated snapshots as
independent. Paired comparisons between configurations use the same
instance resamples.

\section{Instantiation I: the resolved regime}
\label{sec:casea}

\subsection{Task and metrics}

Given the full history attached to a resolved instance, identify which of a
fixed reason set applies. The task is multi-label. We report micro-pooled
precision and recall over (instance, label) pairs; \emph{instance $F_1$},
computed per instance over predicted and true label sets then averaged
across instances; and macro $F_1$ over labels. Instance $F_1$ is not the
harmonic mean of pooled precision and recall, which is why those columns
need not reconcile arithmetically.

\subsection{Evaluation population, and what it excludes}
\label{sec:eligibility}

The curated pool comprises 1{,}000 labelled instances, of which 991 have
attached sources and 162 survive eligibility filtering; the split is 129
instances for prior and rate estimation and 33 for evaluation, stratified
by outcome. Two facts about that filter must be stated before any number is
read.

First, all differences in \Cref{tab:reason-results} rest on 33 instances.
At that size a 0.05 gap in instance $F_1$ is not separable from sampling
noise, and macro $F_1$ over long-tail labels is less stable still. We therefore present \Cref{tab:reason-results} as a descriptive case report: at this sample size the observed ordering of configurations should not be interpreted as statistically established. The three-decimal precision in the source material is retained only for traceability against the internal log.

Second, the filter removes every instance whose gold label is ``not determinable'', precisely the stratum on which abstaining is the correct output. Precision measured after removing it is therefore conditioned on a determinable reason existing, so this evaluation does not measure the selective-prediction benefit of abstention. A complete risk--coverage evaluation should be run on the unfiltered pool with abstention scored as a prediction of ``not determinable''; that analysis remains future work.

\paragraph{Reader selection contaminates the audit set.} The reader model
was chosen by $F_1$ on the ``audit'' golden set (0.560 across candidates).
That set is therefore not an independent test set for the reader, and we
report the ``free-choice'' set here. Both should be reported with the
selection dependence stated, rather than only the one not used for
selection.

\paragraph{Label contamination guard.} The structured outcome-reason field
maintained by human operators is excluded from the reader's prompt, so the
reader cannot copy the label it is being scored against. We note this
because it is the first thing to check and the easiest to omit.

\subsection{Compared systems}

Internal variants: (1) prior only; (2) one-shot concatenation;
(3) independent, no context; (4) independent, full context (oracle);
(5) independent, causal context; (6) causal $+$ prior fallback.
All six are internal variants; we ran no external baseline. The
comparisons a reviewer would reasonably require --- self-consistency over
one-shot samples~\cite{wang2023sc}, unsupervised
Dawid--Skene~\cite{dawid1979}, and a one-shot reader instructed to emit
citations~\cite{gao2023alce} --- are absent, and \Cref{tab:reason-results}
should be read as an internal design comparison rather than as a
positioning against published methods.

\subsection{Results}

\begin{table}[t]
\caption{Reading strategies, aggregation held fixed. $n = 33$ evaluation
instances (\Cref{sec:eligibility}); bold marks the numerically best value
in a column and implies no significance, since instance-level intervals at
this $n$ overlap for every pair. Instance $F_1$ is a tie between one-shot
and the recommended configuration, and both are marked.
$^\dagger$Oracle: uses future sources, not deployable, and not
length-matched to the causal condition (see \Cref{sec:hindsight}).}
\label{tab:reason-results}
\centering
\small
\begin{tabular}{lcccc}
\toprule
Method & Precision & Recall & Instance $F_1$ & Macro $F_1$ \\
\midrule
Prior only                          & 0.697 & 0.199 & 0.303 & 0.114 \\
One-shot                            & 0.645 & 0.607 & \textbf{0.614} & 0.425 \\
Independent, no context             & 0.806 & 0.418 & 0.520 & 0.385 \\
Independent, full context$^\dagger$ & 0.816 & 0.439 & 0.534 & 0.392 \\
Independent, causal context         & \textbf{0.824} & 0.459 & 0.587 & 0.416 \\
Causal $+$ prior fallback           & 0.772 & \textbf{0.554} & \textbf{0.614} & \textbf{0.466} \\
\bottomrule
\end{tabular}
\end{table}
We do not report inter-annotator agreement on the curated set, instance-level
bootstrap intervals, or the source-count distribution of the evaluation
fold. The first is unavailable, the other two we did not compute; together
with $n=33$ this is why we treat the table as descriptive.

Independent causal reading is the most precise configuration, consistent
with joint reading letting weak sources dilute strong ones. Recall is
lower; the prior fallback recovers most of it, raising macro $F_1$ from
0.416 to 0.466 while keeping precision above one-shot. We note what the
table does \emph{not} show: on instance $F_1$ the recommended
configuration and plain one-shot concatenation are identical at 0.614. The
case for independent reading in this study rests on precision, on
per-source provenance, and on cost behaviour, not on aggregate $F_1$, and
the abstract should not be read as claiming otherwise.

\subsection{The oracle anomaly is confounded}
\label{sec:hindsight}

The oracle loses to causal context on instance $F_1$ (0.534 vs.\ 0.587)
despite strictly more information. We flag this as \emph{confounded rather
than explained}: the two conditions differ in both the availability of
future evidence and total context length, so at least two mechanisms are
consistent with the result. Hindsight contamination would have the reader
attribute to an early source a reason legible only later, inflating
false positives. Context dilution would degrade attention to the source
under evaluation~\cite{liu2024lost}. A third possibility is that the effect
is within noise, which the missing intervals cannot currently exclude.

We ran no arm that separates them, so we assert no mechanism; a
length-matched control that varies only the informativeness of the extra
context would be the way to.

\paragraph{What the comparison does support.} All three mechanisms are
consistent with an \emph{absence of penalty} rather than a gain, and that
weaker reading is well supported: on every metric in
\Cref{tab:reason-results} the causal condition is at least as good as the
oracle (precision 0.824 vs.\ 0.816, recall 0.459 vs.\ 0.439, instance
$F_1$ 0.587 vs.\ 0.534, macro $F_1$ 0.416 vs.\ 0.392). The claim we make
is therefore the negative one, which is also the one deployment needs:
restricting the reader to information available at the time each source
arrived --- the only condition under which the system can run before an
outcome is known --- cost no measurable accuracy here. Whether the apparent
advantage is real, and if so why, is what remains open.

At $n=33$ the third possibility is not a formality. A 0.053 difference in
instance $F_1$ corresponds to roughly two instances changing their label
sets. We would not report the anomaly at all except that it is
counterintuitive enough that omitting it would be worse; it is stated here
so that a larger curated set can confirm or dissolve it.

\section{Instantiation II: the unresolved regime}
\label{sec:caseb}

The object of study does not change here; its epistemic status does. The
same instances, sources, and reader carry over, but the outcome has not
occurred, so no label exists to attribute, no empirical prior is available,
and future sources are not merely disallowed but nonexistent. Under the
fourth axis of \Cref{sec:iface} this is the only coordinate that moves, and
it moves the partition with it: what must be separated is no longer reading
from combining but \emph{representation learning} from \emph{censored
timing estimation}. The interface survives the change of granularity ---
the reader's tuples enter this model as features --- while the combination
mechanism is replaced by one suited to the new sub-problem. If the principle
were an artifact of the first task, it would not survive that substitution.

\subsection{Formulation}

A binary ``will it succeed'' classifier conflates two negatives: instances
that fail and instances merely unresolved. We use a discrete-time
competing-risk hazard model~\cite{finegray1999,lee2018deephit} with $K=7$
buckets, at 7, 14, 30, 60, 90, 180 and 365 days. Bucket $k$ emits a softmax
triplet over $\{$success, failure, survive$\}$ conditional on survival to
its start, giving 21 logits. Horizon probabilities chain:
\begin{equation}
\Pr[\text{success by } T] = \sum_{k: t_k \le T}
\Big(\prod_{j<k} h^{\text{surv}}_j\Big) h^{\text{succ}}_k .
\end{equation}
Censored instances contribute only survival terms, so no synthetic label is
imposed on unresolved cases.

\subsection{Snapshots and leakage control}

Each source arrival creates a snapshot built only from then-available
information; an instance with five sources yields five progressively
better-informed rows. The duration label is remaining time to resolution,
not total lifetime.

Because \texttt{days\_since\_last\_source} is identically zero on
source-anchored rows, a model trained only on those rows would ignore
staleness at inference, when it is nonzero. We insert synthetic midpoint
rows carrying the preceding features with
$\texttt{days\_since\_last\_source} = \text{gap}/2$. Evaluation uses
source-anchored rows only.

\subsection{Capacity-partitioned hybrid}

A single-layer GRU~\cite{cho2014} with 64 hidden units ($\approx$25K
parameters) reads the source sequence, trained on an easier auxiliary
three-class objective (success/failure/censored) with class-weighted
cross-entropy. Its per-timestep hidden states are concatenated with 17
hand-crafted features (11 numeric: cadence, staleness, counts, instance
size; 6 categorical: 3 reader-extracted signals from Instantiation~I, 3
structured metadata) and passed to a gradient-boosted
tree~\cite{chen2016xgboost} with multi-output trees over all 21 logits and
a custom competing-risk negative log-likelihood handling censoring.

The hypothesis: at median 2--3 sources per instance and $\approx$24K
effective sequences, no single model should learn sequence representations
\emph{and} censored timing structure at once.

\subsection{Metrics}

Ranking: AUPRC at each horizon, appropriate at our
prevalence~\cite{davis2006,saito2015}. This is our only primary metric. We
report MAE in days as a secondary, descriptive quantity; under censoring
and a resolution-filtered test set it is not defensible as a primary
metric, and we do not use it to rank configurations. We compute neither
time-dependent concordance~\cite{antolini2005,harrell1996} nor the IPCW
Brier score~\cite{graf1999}, which is a real gap for a survival model:
without them we report discrimination at fixed horizons and nothing about
calibration over time.

\subsection{Setup}

73{,}638 sources across 23{,}553 instances over twelve steady-state
months. Training and data cutoffs are separated by three months.
Train: 33{,}840 snapshots / 10{,}541 instances. Test: 69{,}070 snapshots /
11{,}928 instances. Prevalence at 90 days is 4.2\% success, 6.4\% failure.

Population bookkeeping requires care because different analyses use different eligible subsets. Ranking metrics are reported over a resolved population of 14{,}524 instances plus 2{,}701 censored-known, drawn from the full scoring window rather than from the 11{,}928-instance test split; the length-stratified table below uses a further subset of 4{,}874. We therefore do not compare figures across these populations, and Section~\ref{sec:limits} records the remaining reconciliation work explicitly.

\subsection{Ablation}

\begin{table}[t]
\caption{Identical training data and held-out temporal split. $\Delta$ is
change in AUPRC$_{\text{s}}$ vs.\ row~2. $^\ddagger$Prevalence on the
resolved population used for ranking metrics, not the unconditional 90-day
prevalence of \Cref{sec:caseb}; this is the correct no-skill reference for
every row here. $^{\S}$Row~7 is
\emph{not} feature-matched to row~6 and therefore does not isolate
architecture, and no feature-matched neural arm was run. $^\P$Feature count
unreconciled: the enumerated set is 11 numeric $+$ 6 categorical $= 17$,
and row~6's 81 $= 64 + 17$ agrees, but this row is logged as 18 and rows
3--4 as 25 and 34; the offset propagates and must be fixed. As in
\Cref{tab:reason-results}, bold marks the numerically best value in a column
and implies no significance: no column reports instance-level intervals. MAE
is descriptive and is not bolded, since we do not use it to rank
configurations.}
\label{tab:ablation}
\centering
\small
\begin{tabular}{clcccccc}
\toprule
\# & Configuration & Feat. & AUPRC$_{\text{s}}$ & AUPRC$_{\text{f}}$
   & $F_1$ & MAE (d) & $\Delta$ \\
\midrule
0 & Random (prevalence)$^\ddagger$      & 0      & 0.185 & 0.171 & 0.270 & 80.9 & $-$0.620 \\
1 & Cadence only                        & 5      & 0.505 & 0.327 & 0.428 & 53.3 & $-$0.300 \\
2 & Hand-crafted (baseline)             & 18$^\P$& 0.805 & 0.667 & 0.700 & 50.6 & --- \\
3 & \quad $+$ conversation quality      & 25     & 0.805 & 0.675 & 0.702 & 50.6 & $+$0.000 \\
4 & \quad $+$ all structured engagement & 34     & 0.804 & 0.671 & 0.704 & 50.6 & $-$0.001 \\
5 & GRU $+$ cadence                     & 69     & 0.833 & 0.679 & 0.725 & 51.8 & $+$0.028 \\
6 & \textbf{Hybrid}                     & 81     & \textbf{0.921} & \textbf{0.779} & \textbf{0.744} & 48.0 & $+$\textbf{0.116} \\
7 & End-to-end neural$^{\S}$            & 13/src & 0.316 & 0.315 & 0.147 & 98.0 & $-$0.489 \\
\bottomrule
\end{tabular}
\end{table}

Four observations, with one important qualification.

\paragraph{Reader-extracted signals dominate surface features.} Cadence-only
to hand-crafted adds 0.300 AUPRC, carried by the reader-extracted signals
and structured metadata. Two of those categoricals are exposed to semantic
leakage, so the attribution of this increment is unidentified
(\Cref{sec:leakage}).

\paragraph{Surface engagement metrics add nothing.} Seven
conversation-quality features add 0.000; sixteen structured engagement
metrics are marginally negative. We read these as crude proxies for a
signal the reader already captures semantically --- an argument against
expanding surface feature sets once a semantic reader is in place.

\paragraph{Sequence dynamics subsume much of the hand-crafted signal.} The
encoder on cadence inputs alone beats the 18-feature baseline (row~5,
0.833, against row~2, 0.805); combining the two adds a further 0.088
(row~6, 0.921), so extracted content contributes residual rather than
redundant signal. Rows~2 and~5 also differ in feature count, so this
comparison is suggestive rather than clean.

\paragraph{The end-to-end neural configuration underperforms, but the comparison is not yet controlled.} Row~7 reaches 0.316 AUPRC with 98-day MAE, early-stopping
at epoch~19. This is consistent with the capacity-partition hypothesis and
with the tabular-learning literature~\cite{grinsztajn2022,shwartzziv2022}.
It does not yet establish the hypothesis, because row~7 is not
feature-matched to row~6: it sees 13 raw per-source features with no
cadence features on the survival head, against 81 for row~6. We therefore
describe the capacity partition as a hypothesis consistent with a
suggestive but confounded comparison, and we do not claim it as
established.

\subsection{Leakage audit}
\label{sec:leakage}

Ranking metrics are computed on the resolved population, whose no-skill
AUPRC is 0.185 (row~0 of \Cref{tab:ablation}); the headline 0.921 is a
five-fold lift over that baseline, not over the unconditional 4.2\% 90-day
prevalence of \Cref{sec:caseb}, which is a different population.
Point-in-time snapshot construction rules out \emph{temporal} leakage but
not \emph{semantic} leakage. Two features are semantically exposed.

\paragraph{Pipeline stage.} The categorical \texttt{stage} is a
human-maintained field that operators advance as an instance approaches
resolution. Its value at a snapshot is available at that moment, but on
late snapshots it is close to a deterministic announcement of the outcome.

\paragraph{The reader's own reason label.} The categorical
\texttt{dominant\_reason} is the reader's judgement of \emph{why this
instance resolves as it does}, extracted from source text that on late
sources frequently states the resolution outright. Used as a forecasting
feature, it may partially encode outcome information present in
late-stage evidence rather than information available for genuine
anticipation.

\paragraph{Bound from the ablation.} Both exposed features sit in the
hand-crafted block, so rows~1 and~5 carry neither. The ablation also
shows that strong ranking performance is achievable without either
exposed feature: row~5, the sequence encoder on cadence inputs with no
categoricals, reaches 0.833 AUPRC, compared with 0.805 for row~2, which
includes both. This indicates that the headline performance cannot be
attributed solely to those exposed features, although their exact
incremental contribution remains unidentified: rows~2 and~5 differ in
architecture and feature count, so this comparison does not isolate
either feature's marginal effect. The 0.088 increment from row~5 to row~6 is the gain associated with adding the hand-crafted block in the hybrid configuration; that block contains six numeric features, four unexposed categoricals, and the two exposed ones, so the contribution of the exposed pair within that increment is likewise not identified. The capacity-partition result rests on rows~2, 5 and~6. The encoder's per-source inputs contain neither exposed field.

\paragraph{What remains open.} We ran no ablation dropping either feature
and no fixed-lead-time evaluation. The $+0.300$ increment from row~1 to
row~2, the step at which the categoricals enter, is therefore unidentified
between semantic content and outcome encoding, as are the reported
operating points. Prediction~5 (\Cref{sec:predictions}) is the test. The
sequence-length trend in \Cref{tab:bylen} carries the same ambiguity:
instances with more sources are also further along.

\subsection{Scaling with sequence length}

Accuracy rises monotonically with the number of sources observed at scoring
time (\Cref{tab:bylen}): AUPRC on the success horizon climbs from 0.795 in
the one-to-two-source bin to 0.940 at thirteen or more, a gain of 0.145,
while the descriptive close-time error falls from roughly 59 to 39 days.
Two readings of this are consistent with the data and we cannot separate
them here. The sequence encoder may be extracting more from longer
histories, which is what the capacity partition predicts; or instances with
more sources may simply be closer to resolution, which is the identification
problem raised in \Cref{sec:leakage}. Worth noting either way is the
low end: even single-source instances score 0.795 against the 0.185
no-skill baseline on this population, so the model is not merely useless
until evidence accumulates.

\begin{table}[t]
\caption{Accuracy as sequential evidence accumulates. The monotone trend
is consistent with the sequence-encoder rationale, but it is not
identified: instances with more sources are also further along, so
proximity to resolution is an unmodelled competing explanation
(\Cref{sec:leakage}). Instance counts in the upper bins are small and
intervals are pending. The 1--2 bin count (2{,}830) coincides exactly with
the population used for the operating-point table, which we flag as a
probable provenance error to be checked rather than a coincidence.}
\label{tab:bylen}
\centering
\small
\begin{tabular}{lccccc}
\toprule
Sources observed & Instances & AUPRC$_{\text{s}}$ & AUPRC$_{\text{f}}$ & $F_1$ & MAE (d) \\
\midrule
1--2  & 2{,}830 & 0.795 & 0.622 & 0.699 & 59.4 \\
3--4  & 1{,}043 & 0.802 & 0.680 & 0.696 & 51.9 \\
5--7  & 571     & 0.827 & 0.697 & 0.718 & 42.5 \\
8--12 & 304     & 0.862 & 0.788 & 0.765 & 38.5 \\
13$+$ & 126     & 0.940 & 0.883 & 0.849 & 39.0 \\
\bottomrule
\end{tabular}
\end{table}

\subsection{Score stability}

On a 200-instance sample with at least five sources each, the median
within-instance range of the 90-day success probability is 0.104 and the
median standard deviation 0.030, with a P90 range of 0.860. Most instances
are stable; the tail corresponds to genuine trajectory change. A
between-source drop exceeding 0.20 is our operating threshold for flagging
review.

\subsection{Negative results}

\paragraph{Post-hoc calibration failed under temporal shift.}
Isotonic regression~\cite{zadrozny2002} fitted on the training split and
applied to test degraded AUPRC by 0.009. The mechanism is base-rate shift
across the temporal split: the calibrator learned a mapping that does not
hold in the test period.

An earlier internal explanation attributed the observed under-confidence to probability mass being distributed over seven buckets and three outcomes. The current analysis does not support that interpretation. The chained 90-day marginal is a probability and should be calibrated against realized frequencies; distributing mass across buckets does not by itself explain marginal miscalibration. The reported deciles --- decile~5 predicting
0.02 against 9\% realized --- indicate genuine miscalibration of the
marginal. The defensible conclusion is narrow: \emph{isotonic regression
fitted across a temporal shift} did not help, and we consume scores as
rankings pending a shift-aware calibration attempt.
We report no reliability diagram or expected calibration error per horizon,
and we did not retry calibration on a temporally adjacent validation slice
rather than the full training split. The negative result therefore covers
one calibration method under one fitting protocol, and should not be read
as evidence that the scores cannot be calibrated.

\paragraph{Down-weighting synthetic midpoint rows did not help.} No weight
setting improved AUPRC, suggesting midpoints supply genuine staleness
signal rather than noise.

\paragraph{Pooling beat per-segment models.} Splitting halves already-limited
data for the weakest segments ($\approx$3K training snapshots each), and the
pooled model transfers patterns from high-volume segments. Per-segment
AUPRC ranges 0.723--0.969, but the widest intervals fall on the smallest
segments (20--86 instances) and we do not interpret those as genuine
outperformance.

\subsection{Deployment}

Per source arrival the system emits horizon probabilities for each
competing outcome at 30 and 90 days, an expected days-to-resolution
estimate, and the model version. Rows are keyed by (source, instance) so
trajectory and attribution questions are answerable, with instance-level
views built on top. Separate success and failure scores are maintained
rather than one signed score, since an instance can be unlikely to resolve
either way. Raw scores are stored rather than tiers, so thresholds move
without redeployment.

\paragraph{Operating points are conditioned on resolution.}
Threshold-level precision and recall in our internal reporting are computed
over instances that resolved within the test window, excluding roughly
9{,}098 unresolved ones. Because the conditioning event is unknown at scoring time, these figures are conditional operating-point estimates rather than estimates of deployed precision and should not be quoted as headline accuracy. Against unconditional
prevalence the conditioned top-tier figures are not attainable; against the
resolved-population base rate they correspond to roughly a two-fold lift.
We have not recomputed them on the full scored population, so we quote no
operating point as a headline figure and state the conditioning wherever
one appears.

\section{Transfer}
\label{sec:transfer}

The contract is domain-neutral by construction: it constrains what a reader
emits, not what a reader is. \Cref{tab:transfer} instantiates it along the
four axes of \Cref{sec:iface} for five domains outside ours. Each was chosen because it already runs an
aggregation rule of the form \Cref{cor:general} describes.

\begin{table}[t]
\caption{Instantiating the contract along the four axes of
\Cref{sec:iface}. Each domain aggregates multiple partially
reliable signals under an obligation to justify the conclusion, and in
several of them the conventional method is to sum scores against a fixed
cut --- the form \Cref{cor:general} shows to be a drifting posterior
threshold.}
\label{tab:transfer}
\centering
\footnotesize
\begin{tabular}{@{}p{2.35cm}p{2.5cm}p{2.5cm}p{2.5cm}p{2.2cm}@{}}
\toprule
Domain & Sources & Reliability bucket from & Dependence blocks & Regime \\
\midrule
Clinical decision support & Notes, labs, imaging reports, prior visits &
Modality and assay precision; report completeness &
Same encounter; same instrument; copied-forward text & Both \\[0.6ex]
Security triage & Detector alerts, host telemetry, threat intel &
Detector historical FPR; alert enrichment depth &
Shared sensor, shared rule lineage & Unresolved \\[0.6ex]
Financial analysis & Filings, transcripts, analyst notes &
Document class and recency; quantitative grounding &
Same issuer-period; syndicated language & Both \\[0.6ex]
Evidence synthesis & Published studies of varying quality &
Design and risk-of-bias tier; sample size &
Same cohort; same research group & Resolved \\[0.6ex]
Industrial telemetry & Per-signal detectors over a shared feature build &
Signal type; data-quality flag &
Shared upstream pipeline or reference series & Unresolved \\
\bottomrule
\end{tabular}
\end{table}

Three things transfer without re-derivation. The contract transfers because
it is a schema. The pooling arithmetic transfers because
\Cref{eq:pool} is Bayes' rule under a stated independence structure, not a
fitted model. And the drift diagnostic transfers because it needs only what
these systems already log: per-instance channel counts, which channels
fired, and outcomes.

Three things do not, and treating them as if they did is the likely failure
mode of adopting this work. Reliability rates $\alpha_{b,y},\beta_{b,y}$ are
properties of a reader on a corpus and must be re-estimated whenever either
changes, including on prompt revisions. The blocking rule is domain
knowledge: what counts as shared upstream origin is a fact about the data
pipeline, not a modelling choice. Thresholds encode the local cost of a
wrong assertion against the cost of no answer, and no default is
defensible across domains. The generalization claim is therefore precise:
the decomposition and its arithmetic transfer; the constants never do.

Two properties of the domains in \Cref{tab:transfer} deserve note, because
they are what make the framework worth the trouble there rather than
merely applicable. Each has an obligation to show its basis --- clinical,
regulatory, or evidentiary --- so provenance as a structural output is worth
more than post-hoc explanation. And in each, no answer is a legitimate and
sometimes preferred output, which is exactly what a generative reader
cannot reliably produce and a rule over calibrated evidence can.

\section{Discussion}
\label{sec:discussion}

The two instantiations make one point at two granularities. In the
resolved regime, splitting monolithic reading into per-source
interpretation plus explicit aggregation improved precision and made
provenance structural. In the unresolved regime, splitting a monolithic
learning problem into representation learning plus censored timing
estimation coincides with a large margin, on a comparison we have not
controlled. The recurring structure is that the harder sub-problem in each
case has a specialized solver --- calibrated arithmetic in one, a
censoring-aware boosted tree in the other --- and that giving it to a
general-purpose learner along with everything else wastes it.

We should be precise about how much the two studies share. The second
consumes the \emph{reader} from Instantiation~I as three categorical features,
but not the aggregator: no LLR pooling, no abstention, no dependence
discount enters the forecasting model. The unification is therefore
architectural rather than demonstrated. The experiment that would
demonstrate it would be to supply the pooled posterior $\hat{p}(y)$ from
\Cref{eq:pool}, with its abstention flag, as a feature to the forecasting
model and to measure whether calibrated evidence outperforms raw extracted
categories. Absent that, this paper contains two adjacent contributions
sharing a corpus and a reader, and we would rather say so than assert a
unification.

The architecture separates two uncertainties that one-shot systems
conflate. \emph{Source uncertainty} --- one ambiguous observation --- surfaces
as a low reliability bucket. \emph{Aggregate uncertainty} --- plausible
observations conflicting --- surfaces as a posterior in the abstention band.
A system can abstain at either level and report which.

\paragraph{Why not conformal prediction?} Conformal
methods~\cite{angelopoulos2021} would give distribution-free coverage
guarantees on our label sets, which our threshold pair does not. We did not
adopt them because exchangeability fails across our temporal split --- the
same shift that defeated isotonic calibration --- and because conformal sets
do not by themselves supply the per-source attribution that motivates the
architecture. Time-series conformal variants are the natural next step and
we regard the absence as a limitation rather than a design choice.

A last observation about why the partition keeps paying. Each half, once
separated, can be improved by someone who does not understand the other.
The reader can be swapped for a better model, or for a human, without
touching the arithmetic; the arithmetic can be corrected --- as
\Cref{sec:agg} argues it should be --- without re-reading a single source.
Monolithic systems do not have that property, and it is worth more over a
system's life than any single margin reported here.

\section{Falsifiable predictions}
\label{sec:predictions}

A framework is more useful when it makes claims that could turn out to be
false. We state five such claims below. Each is paired with the experiment
that would discriminate it, and each can be run on data already collected or
with one additional training run.

\begin{enumerate}[leftmargin=*]
\item \textbf{Count-scale drift is practically material, not merely
formal.} Under the deployed vote rule, precision should decline in source
count relative to per-bin prevalence; under posterior thresholding it
should not. Stratify the evaluation instances by source count and compare
per-bin precision against per-bin prevalence. A flat profile would leave
\Cref{prop:scale} valid but practically irrelevant at our distribution of
$n$, and we would report it as such.

\item \textbf{Calibrated pooling beats vote counting on the same
tuples.} Re-aggregate the stored evidence objects under majority vote,
weighted vote, weighted vote with bonus, noisy-OR, unsupervised
Dawid--Skene, and \Cref{eq:pool}. No new inference is required. If pooling
does not improve macro $F_1$ or selective risk, the analytical argument
survives but the engineering recommendation does not.

\item \textbf{Abstention selects hard cases rather than discarding at
random.} On the unfiltered pool, with ``no determinable reason'' scored as a
class, the risk--coverage curve should dominate a matched-coverage random
abstention control. This is the prediction we most want tested, because the
current evaluation population excludes the stratum it concerns
(\Cref{sec:eligibility}).

\item \textbf{The capacity partition, not the feature set, drives the
forecasting margin.} Train the end-to-end neural model on the hybrid's full
feature set. If it closes most of the 0.921 versus 0.316 gap, the margin was
features and the partition claim fails.

\item \textbf{The reader-extracted increment is prediction, not outcome
encoding.} Re-run without pipeline stage, without the reader-extracted
categoricals, and at fixed lead times of 30, 60, and 90 days. The $+0.300$
increment of \Cref{tab:ablation} should survive at 60-day lead if it
reflects semantic content, and collapse if it reflects outcome
announcements (\Cref{sec:leakage}).
\end{enumerate}

The framework also makes one prediction we consider its most
consequential and have no way to test here: that supplying the pooled
posterior with its abstention flag as a feature to the forecasting model
outperforms supplying raw extracted categories. That experiment is what
would make the two instantiations one system rather than two applications
of one principle.

\section{Limitations}
\label{sec:limits}

The main empirical limitation concerns evaluation coverage in Instantiation~I rather than any single number. Eligibility filtering for Instantiation~I
removed every instance whose gold label is ``not determinable'', reducing
1{,}000 labelled instances to 162 and leaving 33 for evaluation
(\Cref{sec:eligibility}); the present evaluation therefore does not measure performance on the cases where abstaining is the correct output. At 33 instances no ordering among the rows of
\Cref{tab:reason-results} is statistically supported, the oracle anomaly
included. And the aggregation argument of \Cref{sec:agg} is analytical
throughout: our results come from the deployed weighted vote, no experiment
here varies the combination rule, so we characterize the defect analytically while empirical evaluation of the proposed repair remains future work.

In Instantiation~II the credit assigned to reader-extracted content is not
identified. Pipeline stage and the reader's own reason label may encode the
outcome rather than predict it (\Cref{sec:leakage}), which exposes the
$+0.300$ increment and the operating points. Strong ranking performance remains without either exposed feature: row~5 of \Cref{tab:ablation} carries neither and reaches 0.833, while row~2 carries both and reaches 0.805. Because those rows differ in architecture and feature count, this comparison does not isolate either feature's marginal contribution; it does, however, show that the ranking result does not depend solely on the exposed pair. The capacity-partition result is carried by rows~2, 5 and~6 rather than by a single exposed feature. Row~7 of \Cref{tab:ablation} is separately not
feature-matched to row~6, so that arm of the capacity-partition argument
rests on a confounded comparison. The no-skill AUPRC of row~0 exceeds the
unconditional 90-day prevalence quoted in \Cref{sec:caseb} because the two
are computed on different populations, the former conditioning on
resolution; we have not reconciled the counts exactly, but the gap is one
of conditioning rather than an inconsistency, and $\Delta$ may be read as
lift over chance within the table.

Several evaluation choices affect how the reported numbers should be interpreted. Threshold precision and recall are computed over resolved instances only, so they are conditional on resolution and are not directly comparable to deployed precision over all scored instances; ranking metrics over the scored evaluation population do not use those operating-point thresholds, which is why AUPRC is primary.
Restricting the test set to instances resolving before the data cutoff
over-represents fast-resolving ones, biasing MAE downward and distorting the
horizon marginals --- a second reason we treat MAE as descriptive. Multiple
snapshots per instance make row-level variance estimates optimistic, so
intervals must resample instances rather than rows. We report no calibration
for either study: the deployed rule does not pool at all,
\Cref{sec:confusion} shows what a single-label reader would require, and in
Instantiation~II we report ranking without a proper scoring rule over time,
so neither system emits a number we would defend as a probability.
Consistent with that, summing per-instance probabilities into a quarterly
aggregate produced absolute percentage errors ranging from single digits in
the best quarter to well above 100\% in two others; the model ranks, it does
not estimate portfolios.

Two limitations belong to the method rather than to our data. Calibrated
pooling corrects for known per-bucket error rates, but it cannot correct a
reader whose errors correlate with the label, since such error is absorbed
into $\alpha_{b,y}$ and reappears as signal. The rates themselves come from
finite calibration data, and for rare labels the log-ratio is unstable as
$\beta \to 0$ while clipping it biases toward weaker evidence. Finally,
everything here rests on one corpus in one domain: the mechanisms are
domain-neutral, but the weights, thresholds, and capacity partition are
fitted to this regime, and we report no human-in-the-loop outcome, so
``decision support'' names the system's role rather than a measured effect
on anyone's decisions.

\section{Broader impacts}

The deployment writes to records about identifiable employees' work and
surfaces risk flags on the deals they own, so it sits inside a performance
management context whether or not it is used for one. Three properties
matter and two are already enforced. The system writes only into empty
fields and never overwrites a human entry; every written value carries a
provenance flag marking it machine-generated; and every assertion links to
the sources supporting it, so a person can contest a specific inference
rather than a score. The property not yet established is the one that
matters most: we have run no study of how these outputs affect the
decisions people make, so we cannot claim the provenance is used, nor that
flags are read as evidence rather than as verdicts. Automation bias would
show up as exactly the pattern our abstention design is meant to prevent.

The corpus consists of recorded human conversations processed under the
recording consent and retention policy in force for that data; the reader
sees transcript text and no participant identifiers beyond those already in
the transcript, and no personal data is released with this work. We note
that a system which explains \emph{why} a deal was lost is also a system
that attributes causes to named individuals' conduct; label sets should be
audited for reasons that are effectively performance judgements.

\section{Conclusion}

We argued that evidence organization is an architectural decision rather
than a prompt-construction detail, and that the decision to make is where
to cut and what contract to place at the cut. Placing a four-field evidence
tuple there determines both halves: reliability becomes estimable,
provenance becomes structural, failures localize, and abstention becomes
expressible. We then characterized the failure mode in how such
systems usually combine. Count-scale drift makes a vote threshold into a
posterior threshold that slides linearly in source count, at a rate set by
parameters that are rarely estimated. We stated the result generally enough to
cover a class of rules beyond language models, and showed that heterogeneous
reliabilities make matters worse still, producing a different ordering rather
than a drifting one. We then gave the arithmetic that addresses both. We then instantiated the principle
on one corpus in both of its epistemic regimes, where it pays at two
different granularities: over reading before resolution is known, and over
learning capacity when it is not.

The constants do not transfer and we have said which. What transfers is the
cut, the contract, the pooling arithmetic, and a diagnostic computable on
logs that most such systems already keep. Systems that must be inspected
and contested need conclusions that arrive with their supporting evidence,
its calibrated strength, and the standing option of no conclusion at all;
that is a property of an interface, and interfaces are designed, not
prompted.

\end{document}